\documentclass[11pt]{article}
\usepackage{amsmath}
\usepackage{latexsym}
\usepackage{amssymb}
\usepackage[square,numbers]{natbib}
\usepackage{xspace}
\usepackage{algorithm,algorithmic}
\usepackage{hyperref}
\usepackage{fullpage}

\newtheorem{Thm}{Theorem}
\newtheorem{Lem}{Lemma}

\newenvironment{proof}{\noindent {\sc Proof:}}{$\Box$ 
        }

\newcommand{\E}{\ensuremath{\mathbb{E}}}

\newcommand{\bigP}[1]{\left(#1\right)}
\newcommand{\bigB}[1]{\left[#1\right]}
\def\eps{\varepsilon}

\begin{document}
\title{Multiarmed Bandits With Limited Expert Advice}
\author{Satyen Kale \\
IBM T. J. Watson Research Center\\
Yorktown Heights, NY 10598\\
\texttt{sckale@us.ibm.com}}
\date{}

\maketitle
\def\hell{\hat{\ell}}
\def\bit{B_{I_t}}
\def\onevec{\vec{1}}
\def\ind{\mathbf{I}}
\def\hxi{\hat{\xi}}
\def\be{\mathbf{e}}
\def\mK{\mathcal{K}}
\def\mN{\mathcal{N}}
\def\mR{\mathcal{R}}
\def\hstar{{h^\star}}
\def\bernoulli{\mathbb{B}}
\def\half{\tfrac{1}{2}}
\def\bP{\mathbf{P}}
\def\bQ{\mathbf{Q}}
\def\cinv{\tfrac{1}{c}}
\newcommand\kl[2]{\text{KL}(#1\ \|\ #2)}
\newcommand\dtv[2]{d_\text{TV}(#1, #2)}
\def\MW{\textsc{MW}\xspace}
\def\INF{\textsc{PolyINF}\xspace}

\begin{abstract}
We solve the COLT 2013 open problem of \citet{SCB} on minimizing regret in the setting of advice-efficient multiarmed bandits with expert advice. We give an algorithm for the setting of $K$ arms and $N$ experts out of which we are allowed to query and use only $M$ experts' advices in each round, which has a regret bound of $\tilde{O}\bigP{\sqrt{\frac{\min\{K, M\} N}{M} T}}$ after $T$ rounds. We also prove that any algorithm for this problem must have expected regret at least $\tilde{\Omega}\bigP{\sqrt{\frac{\min\{K, M\} N}{M}T}}$, thus showing that our upper bound is nearly tight.
\end{abstract}

\section{Introduction}

Consider the following advice-efficient setting of the multiarmed bandits with expert advice problem, introduced by \citet{SCB}. In each round $t = 1, 2, \ldots, T$, we are required to pull one arm $A_t \in \{1, 2, \ldots, K\} =: \mK$. Simultaneously, an adversary sets losses $\ell_t(a) \in [0, 1]$ for each arm $a \in \mK$. Assisting us in this task are $N$ experts in the set $\mN = \{1, 2, \ldots, N\}$. Each expert $h$ can provide advice on which arm to pull in the form of a probability distribution $\xi_t^h$ on the set of arms. This advice gives the expert $h$ an expected loss of $\xi_t^h \cdot \ell_t$ in round $t$. The catch is that we can only observe the advice of at most $M$ experts of our choosing in each round. The goal is to choose subsets of $M$ experts in each round to query the advice of, and using their advice play some arm $A_t \in \mK$ (probabilistically, if desired) to minimize the expected regret with respect to the loss of the best expert, where the regret is defined as:
\[\text{Regret}_T\ :=\ \sum_{t=1}^T \ell_t(A_t) - \min_{h \in \mN} \sum_{t=1}^T \xi_t^h \cdot \ell_t.\]
In the following sections we give an algorithm whose expected regret is bounded by 
\[\sqrt{\frac{2\min\{K, M\} N \log(N)}{M} T}\] 
after $T$ rounds, based on the Multiplicative Weights (\MW) forecaster for prediction with expert advices~\citep{LW}. We can improve this upper bound using the \INF forecaster of \citet{AB} to
\[ 4\sqrt{\frac{\min\{K, M\}N\log(\tfrac{8M}{\min\{K, M\}})}{M} T}.\]
This matches the regret of the best known algorithms for the special cases $M = 1$ and $M = N$, and interpolates between them for intermediate values of $M$. This solves the COLT 2013 open problem proposed by \citet{SCB}, and in fact gives a better regret bound than the bound conjectured in \citep{SCB}, which was $O\left(\sqrt{\frac{KN \log(N)}{M} T}\right)$.

Furthermore, we also show that any algorithm for the problem must incur expected regret of $\Omega\bigP{\sqrt{\frac{\min\{K, \frac{M}{\log(K)}\} N}{M}T}}$ on some sequence of expert advices and arm losses, thus showing that our upper bound is nearly tight.

\section{Preliminaries}

For any event $E$, let $\ind[E]$ be the indicator random variable set to $1$ if $E$ happens. In any round $t$ of the algorithm, let $\Pr_t[\cdot]$ and $\E_t[\cdot]$ denote probability and expectation respectively conditioned on all the randomness defined up to round $t-1$. For two probability distributions $\bP$ and $\bQ$ defined on the same space let $\kl{\bP}{\bQ}$ and $\dtv{\bP}{\bQ}$ denote the KL-divergence and total variation distance between the two distributions respectively.

Without loss of generality, we may assume that each expert suggests exactly one arm to play in any round; i.e. $\xi_t^h(a) = 1$ for exactly one arm $a \in \mK$ and $0$ for all other arms. Call such advice vectors ``standard basis vectors''. To see this, for every expert $h$ we can randomly round a general advice vector $\xi_t^h$ to a standard basis vector by sampling some arm $a_h \sim \xi_t^h$ and constructing a new advice vector $\hxi_t^h$ by setting $\hxi_t^h(a_h) = 1$ and $\hxi_t^h(a) = 0$ for all $a \neq a_h$. Note that in $\E[\hxi_t^h] = \xi_t^h$; thus for any expert $h$ following the randomly rounded advices $\hxi_t^h$ for $t = 1, 2, \ldots, T$ has the same expected cost as following the advices $\xi_t^h$. Since this randomized rounding trick can be applied to the advices (algorithmically for the observed advices, and conceptually for the unobserved advices), in the rest of the paper we assume that all advice vectors are standard basis vectors; this helps us in getting a tighter bound on the regret.

For any time period $t$ and any set $U \subseteq \mN$, define the ``active set of arms'' to be the set of all arms recommended  by experts in $U$, i.e.
\[\mK_t^U = \{a \in \mK:\ \exists h \in U \text{ s.t. } \xi_t^h(a) = 1\}.\]
Note that since we are allowed to query at most $M$ experts in any round, if $U$ is the queried set of experts in round $t$, then $|\mK_t^U| \leq \min\{K, M\}$; this leads to $\min\{K, M\}$ factor in the regret bound. Define $K' := \min\{K, M\}$, the effective number of arms.

\section{Algorithm}

Assume $M$ divides $N$, and partition the $N$ experts into $R = N/M$ groups of $M$ experts each arbitrarily. Call the groups $B_1, B_2, \ldots, B_R$, and define $\mR := \{1, 2, \ldots, R\}$. Run an algorithm for prediction with expert advice (such as Multiplicative Weights (\MW) forecaster of \citet{LW}, or the \INF forecaster of \citet{AB}) on all the experts, where the loss of expert $h$ at time $t$ is given as
\[ Y_t^h\ :=\ \xi_t^h \cdot \hell_t^h,\]
where $\xi_t^h$ is the probability distribution over the $K$ arms specified by expert $h$ at time $t$, $\hell_t^h$ is an estimator for the losses of the arms (we will specify this later; we will ensure that $\hell_t^h = 0$ for all but $M$ experts so that only $M$ experts need to be queried for their advice). 

Let the distribution over experts generated by the expert learning algorithm at time $t$ be $q_t$. Define the probability distribution $r_t$ over group indices $\mR$ as $r_t(i) = \sum_{h \in B_i} q_t(h)$. Each group $B_i$ defines a probability distribtion over arms:
\[p_t^i\ :=\ \frac{\sum_{h \in B_i} q_t(h)\xi_t^h}{\sum_{h \in B_i} q_t(h)}\ =\ \frac{\sum_{h \in B_i} q_t(h)\xi_t^h}{r_t(i)}.\]

Sample $I_t$ from $r_t$, and $A_t$ from $p_t^{I_t}$. Play $A_t$ and observe its loss $\ell_t(A_t)$. For every group $B_i$, define the loss estimator given by
\begin{equation} \label{eq:loss-estimator}
	\hell_t^{i}(a)\ :=\ 
\begin{cases}
	\ell_t^i(a)\frac{\ind[I_t = i, A_t = a]}{\Pr_t[i, a]} & \text{ if } \Pr_t[i, a] > 0\\
	0 & \text{ otherwise},	
\end{cases}
\end{equation}
where
\[\Pr_t[i, a]\ =\ r_t(i)p_t^{i}(a)\] 
is the probability of the event $\{I_t = i, A_t = a\}$, conditioned on all the randomness up to round $t-1$.

For all experts $h \in B_i$, define the loss estimator:
\[\hell_t^h\ :=\ \hell_t^i.\]
Note that except for $h \in \bit$, all $\hell_t^h$ are zero, and for $\bit$, the probabilities $\Pr_t[I_t, a]$ for all arms $a$ can be computed using the only the advices of the experts $h \in \bit$. Thus $Y_t^h$ for all experts $h$ can be computed and the algorithm is well-defined.

\section{Analysis}

We first prove a number of utility lemmas. The first lemma shows that the loss estimators we construct are unbiased for all experts with positive probability (and an underestimate in general):
\begin{Lem} \label{lem:unbiased} For all rounds $t$ and all experts $h$,
	\[ \E_t[Y_t^h]\ \leq\ \xi_t^h \cdot \ell_t\]
	 with equality holding if $q_t(h) > 0$.\footnote{It is easy to see that both the \MW and \INF forecasters always have positive probability on all experts, so if we use one of these two expert learning algorithms, then the all the inequalities in this lemma are actually equalities.} Thus, $\E_t[q_t^hY_t^h] = q_t^h(\xi_t^h \cdot \ell_t)$, and unconditionally, $\E[Y_t^h] \leq \xi_t^h \cdot \ell_t$.
\end{Lem}
\begin{proof} Let $a$ be the arm recommended by expert $h$ at time $t$, i.e. $\xi_t^h(a) = 1$. Note that $Y_t^h = \hell_t^h(a)$ and $\xi_t^h \cdot \ell_t = \ell_t(a)$. Let $h \in B_i$. If $\Pr_t[i, a] > 0$, then by the definition of the loss estimator in (\ref{eq:loss-estimator}), we have
	\[ \E_t[Y_t^h]\ =\ \E_t[\hell_t^h(a)]\ =\ \E_t\left[\ell_t(a)\frac{\ind[I_t = i, A_t = a]}{\Pr_t[i, a]} \right]\ =\ \ell_t(a)\frac{\Pr_t[i, a]}{\Pr_t[i, a]}\ =\ \ell_t(a).\] 
	If $\Pr_t[i, a] = 0$, then $\hell_t^h(a) = 0$, and so $\E_t[\hell_t^h(a)] = 0 \leq \ell_t(a)$. Thus in either case, $\E_t[Y_t^h] \leq \xi_t^h \cdot \ell_t$. Finally, note that if $q_t(h) > 0$, then $\Pr_t[i, a] > 0$, so equality holds. 
\end{proof}

The next lemma shows that in expectation, the loss of the algorithm in each round is the same as the loss of playing an action recommended by sampling an expert from the distribution generated by the expert learning algorithm:
\begin{Lem} \label{lem:small-bias} For all rounds $t$ we have
	\[\E[\ell_t(A_t)]\ =\ \E[\sum_h q_t(h) Y_t^h].\]	
\end{Lem}
\begin{proof}
	\begin{align*}
		\E_t[\ell_t(A_t)]\ =\ \sum_{i \in \mR} \sum_{a \in \mK} \ell_t(a)\Pr_t[i, a]
		&=\ \sum_{i \in \mR} \sum_{a \in \mK} \ell_t(a)r_t(i)p_t^{i}(a)\\
		&=\ \sum_{i \in \mR} r_t(i) (p_t^i \cdot \ell_t)
		\ =\ \sum_{i \in \mR} \sum_{h \in B_i} q_t^h (\xi_t^h \cdot \ell_t) 
		\ =\ \E_t[\sum_h q_t^h Y_t^h],
	\end{align*}
	by Lemma~\ref{lem:unbiased}. Taking expectation over all the randomness up to time $t-1$, the proof is complete.
\end{proof}

The next lemma gives a bound on the variance of the estimated losses. We state this in slightly more general terms than necessary to unify the analysis of the algorithms using the \MW or \INF forecasters as the expert learning algorithm.
\begin{Lem} \label{lem:variance} Fix any $\alpha \in [1, 2]$. For all rounds $t$ we have
	\[\E[\sum_h (q_t(h))^\alpha(Y_t^h)^2]\ \leq\ (RK')^{2 - \alpha}.\]
\end{Lem}
\begin{proof}
	Let 
	\[S\ :=\ \{(i, a) \in \mR \times \mK \ |\ \Pr_t[i, a] > 0\}\] be the set of all (group index, action) pairs that have positive probability in round $t$. Since in round $t$, the algorithm only plays arms in $\mK_t^{\bit}$, and for any group $B_i$, the set of active arms in round $t$, $\mK_t^{B_i}$, has size at most $K'$, we conclude that $|S| \leq RK'$.

	The pair $(I_t, A_t)$ computed by the algorithm is in $S$. Conditioing on the value of $(I_t, A_t)$, we can upper bound $\sum_h (q_t(h))^\alpha(Y_t^h)^2$ as follows:
	\begin{align}
		\sum_h (q_t(h))^\alpha(Y_t^h)^2\ &=\ \sum_{h \in \bit} (q_t(h))^\alpha(\xi_t^h \cdot \hell_t^{I_t})^2 \notag \\
		&=\ \sum_{h \in \bit} (q_t(h))^\alpha\left(\xi^h(A_t) \cdot \frac{\ell_t(A_t)}{\Pr_t[I_t, A_t]}\right)^2 & (\because \hell_t^{I_t}(a) = 0 \text{ for all } a \neq A_t)\notag \\
		&\leq\ \sum_{h \in \bit} (q_t(h)\xi^h(A_t))^\alpha \left(\frac{1}{\Pr_t[I_t, A_t]}\right)^2 & (\because \xi^h(A_t) \in [0,1],\ \alpha \leq 2,\ \ell_t(A_t) \in [0, 1]) \notag \\
		&\leq\ \bigP{\sum_{h \in \bit} q_t(h)\xi^h(A_t)}^\alpha \left(\frac{1}{\Pr_t[I_t, A_t]}\right)^2 & (\because \|\cdot\|_{\alpha} \leq \|\cdot\|_1 \text{ since } \alpha \geq 1) \notag\\
		&=\ \bigP{r_t(I_t)p_t^{I_t}(A_t)}^\alpha \cdot \frac{1}{\Pr_t[I_t, A_t]^2} & \left(\because p_t^{I_t}(A_t) = \frac{\sum_{h \in \bit} q_t(h)\xi^h(A_t)}{r_t(I_t)}\right) \notag \\
		&=\ \Pr_t[I_t, A_t]^{\alpha - 2}, \label{eq:prob-bound}
	\end{align}
	since $\Pr_t[I_t, A_t] = r_t(I_t)p_t^{I_t}(A_t)$. Next, we have
	\begin{align*}
		\E_t[\sum_h (q_t(h))^\alpha (Y_t^h)^2]\ &=\ \E_t[\E_t[\sum_h (q_t(h))^\alpha(Y_t^h)^2\ |\ (I_t, A_t)]\\
		&\leq\ \sum_{(I_t, A_t) \in S} \Pr_t[I_t, A_t] \cdot \Pr_t[I_t, A_t]^{\alpha - 2} & \text{(By (\ref{eq:prob-bound}))}\\
		&=\ \sum_{(I_t, A_t) \in S} \Pr_t[I_t, A_t]^{\alpha - 1}\\
		&\leq\ \bigP{\sum_{(I_t, A_t) \in S} \Pr_t[I_t, A_t]}^{\alpha - 1} \cdot \bigP{\sum_{(I_t, A_t) \in S} 1}^{2 - \alpha} \\
		&=\ |S|^{2 - \alpha}\\
		&\leq\ (RK')^{2 - \alpha}.
	\end{align*}
	The penultimate inequality follows by applying H\"{o}lder's inequality to the pair of dual norms $\|\cdot\|_{\frac{1}{\alpha - 1}}$ and $\|\cdot\|_{\frac{1}{2 - \alpha}}$. Taking expectation over all the randomness up to time $t-1$, the proof is complete.
\end{proof}

\subsection{Analysis using the \MW forecaster}

The \MW forecaster for prediction with expert advice takes one parameter, $\eta$. It starts with $q_1$ being the uniform distribution over all experts, and for any $t \geq 1$, constructs the distribution $q_{t+1}$ using the following update rule:
\[ q_{t+1}(h)\ :=\ q_t(h)\exp(-\eta Y_t^h)/Z_t,\]
where $Z_t$ is the normalization constant required to make $q_{t+1}$ a distribution, i.e. $\sum_h q_{t+1}^h = 1$.

\begin{Thm} \label{thm:mw} Set $\eta = \sqrt{\frac{M \log(N)}{K'NT}}$. Then the expected regret of the algorithm using the \MW forecaster is bounded by $\sqrt{\frac{2K'N \log(N)}{M} T}$.
\end{Thm}
\begin{proof} 
The \MW forecaster guarantees (see \citep{AHK-MW}) that as long as $Y_t^h \geq 0$ for all $t, h$, we have for any expert $h^\star$
	\begin{equation} \label{eq:mw-bound}
		\sum_{t=1}^T \sum_h q_t(h) Y_t^h\ \leq\ \sum_t Y_t^{h^\star} + \frac{\eta}{2} \sum_t \sum_h q_t(h) (Y_t^h)^2 + \frac{\log N}{\eta}.
	\end{equation}
	Now, we have for any expert $h^\star$
	\begin{align*}
		\sum_t \E[\ell_t(A_t)]\ &=\ \sum_t \E[\sum_h q_t(h) Y_t^h]  &  \text{(By Lemma~\ref{lem:small-bias})}\\
		&\leq\ \sum_t \E[Y_t^{h^\star}] + \frac{\eta}{2} \sum_t \E[\sum_h q_t(h) (Y_t^h)^2] + \frac{\log N}{\eta} & \text{(By (\ref{eq:mw-bound}))}\\
		&\leq\ \sum_t \xi_t^{h^\star} \cdot \ell_t + \frac{\eta}{2} RK'T + \frac{\log N}{\eta}  & \text{(By Lemma~\ref{lem:unbiased} and Lemma~\ref{lem:variance} with $\alpha = 1$)}\\
		&\leq\ \sum_t \xi_t^{h^\star} \cdot \ell_t + \sqrt{\frac{2K'N \log(N)}{M} T},
	\end{align*}
	using $\eta = \sqrt{\frac{2\log(N)}{RK'T}} = \sqrt{\frac{2M \log(N)}{K'NT}}$.
\end{proof}

\subsection{Analysis using the \INF forecaster}

The \INF forecaster for prediction with expert advice takes two parameters, $\eta$ and $c > 1$. It starts with $q_1$ being the uniform distribution over all experts, and and for any $t \geq 1$, constructs the distribution $q_{t+1}$ as follows:
\[ q_{t+1}(h) = \frac{1}{[\eta(\sum_{\tau=1}^t Y_\tau^h + C_{t+1})]^c}\]
where $C_{t+1}$ is a constant chosen so that $q_{t+1}$ is a distribution, i.e. $\sum_h q_{t+1}^h = 1$.

\begin{Thm} \label{thm:inf} Set  $c = \log(\tfrac{8M}{K'})$ and $\eta = 2N^{\frac{1}{2c}}[c(RK')^{1 - \cinv}T]^{-\frac{1}{2}}$. Then the expected regret of the algorithm using the \INF forecaster is bounded by $4\sqrt{\frac{K'N\log(\tfrac{8M}{K'})}{M} T}$.
\end{Thm}
\begin{proof}
	\citet{ABL} prove that for the \INF forecaster, as long as $Y_t^h \geq 0$ for all $t, h$, we have for any expert $h^\star$:
	\begin{equation} \label{eq:inf-bound}
		\sum_{t=1}^T \sum_h q_t(h) Y_t^h\ \leq\ \sum_t Y_t^{h^\star} + \frac{c\eta}{2} \sum_t \sum_h (q_t(h))^{1 + \cinv} (Y_t^h)^2 + \frac{cN^{\cinv}}{\eta(c-1)}.
	\end{equation}
	Now, we have for any expert $h^\star$
	\begin{align*}
		\sum_t \E[\ell_t(A_t)]\ &=\ \sum_t \E[\sum_h q_t(h) Y_t^h]  &  \text{(By Lemma~\ref{lem:small-bias})}\\
		&\leq\ \sum_t \E[Y_t^{h^\star}] + \frac{c\eta}{2}\sum_t \E[\sum_h (q_t(h))^{1 + \cinv} (Y_t^h)^2] + \frac{2N^{\cinv}}{\eta} & \text{(By (\ref{eq:inf-bound}), using $c \geq 2$)}\\
		&\leq\ \sum_t \xi_t^{h^\star} \cdot \ell_t + \frac{c\eta}{2} (RK')^{1 - \cinv}T + \frac{2 N^{\cinv}}{\eta} \\
		& \qquad \qquad \text{(By Lemma~\ref{lem:unbiased} and Lemma~\ref{lem:variance} with $\alpha = 1 + \cinv$)}\\
		&\leq\ \sum_t \xi_t^{h^\star} \cdot \ell_t + 2\sqrt{cRK'\bigP{\tfrac{N}{RK'}}^{\cinv} T}, & 
		\text{(Using $\eta = 2N^{\frac{1}{2c}}[c(RK')^{1 - \cinv}T]^{-\frac{1}{2}}$)}\\
		&\leq\ \sum_t \xi_t^{h^\star} \cdot \ell_t + 4\sqrt{\frac{K'N\log(\tfrac{8M}{K'})}{M} T},
	\end{align*}
	using $c = \log(\tfrac{8M}{K'}) = \log(\tfrac{8N}{RK'})$.
\end{proof}

\section{Lower Bound}

In this section, we show a lower bound on the regret of any algorithm for the multiarmed bandit with limited expert advice setting which shows that our upper bound is nearly tight. To describe the lower bound, consider the well-studied balls-into-bins process. Here $M$ balls are tossed randomly into $K$ bins. In each toss a bin is chosen uniformly at random from the $K$ bins independently of other tosses. Define the function $f(K, M)$ to be the expected number of balls in the bin with the maximum number of balls. It is well-known (see, for example, \cite{raab-steger}) that $f(K, M) = O(\max\{\log(K), \frac{M}{K}\})$.

With this definition, we can prove the following lower bound. Note that this lower bound doesn't immediately follow from a similar lower bound from \citet{SPCA} because in their setting the experts' losses can be all uncorrelated, whereas in our setting the experts' losses are necessarily correlated because there are only $K$ arms.
\begin{Thm}
	For any algorithm for the multiarmed bandit with limited expert advice setting, there is a sequence of expert advices and losses for each arm so that the expected regret of the algorithm is at least $\Omega\bigP{\sqrt{\frac{N}{f(K, M)} T}} = \Omega\bigP{\sqrt{\frac{\min\{K, \frac{M}{\log(K)}\} N}{M}T}}$.
\end{Thm}
\begin{proof}
The lower bound is based on the fairly standard information theoretic arguments that originated in \citep{ACFS}. Let $\bernoulli(p)$ be the Bernoulli distribution with parameter $p$, i.e. $1$ is chosen with probability $p$ and $0$ with probability $1 - p$.

In the following, we assume the online algorithm is deterministic (the extension to randomized algorithms is easy by conditioning on the random seed of the algorithm). Fix the parameter
\[ \eps\ :=\ \frac{1}{8}\sqrt{\frac{N}{f(K, M)T}}.\]
The expert advices and the rewards of the arms are generated randomly as follows. We define $N$ probability distributions, $\bP_h$ for all $h \in \mN$. Fix an $\hstar \in \mN$, and we define $\bP_\hstar$ as follows. In each round $t$, for all experts $h \in \mN$, we set their advice to be a uniformly random arm in $\mK$. Denote the arm chosen by expert $h$ in round $t$ by $h(t)$. Conditioned on the choice of the arm $\hstar(t)$, the loss of arm $\hstar(t)$ is chosen from $\bernoulli(\half - \eps)$, and the loss of all arms $a \neq \hstar(t)$ from $\bernoulli(\half)$, independently. Unconditionally, the distribution of the loss of any arm $a$ at any time is $\bernoulli(p)$ where $p = \tfrac{1}{K}\cdot\bigP{\half - \eps} + \tfrac{K-1}{K}\cdot\half = \half - \tfrac{\eps}{K}$. A similar calculation shows that for all experts $h \neq \hstar$, the distribution of the loss of their chosen arm is $\bernoulli(p)$ and thus has expectation $p$, and the expected loss of the arm chosen by $\hstar$ is $\half - \eps$. Thus the best expert is $\hstar$. Let $\E_\hstar$ denote expectation under $\bP_\hstar$.

Consider another probability distribution $\bP_0$ of advices for the experts and losses for the arms: in all rounds $t$, all experts choose their arms in $\mK$ uniformly at random as before, and all arms have loss distributed as $\bernoulli(p)$. Let $\E_0$ denote the expectation of random variables under $\bP_0$.

Before round $1$, we choose an expert $\hstar \in \mN$ uniformly at random, and advices and losses are then generated from $\bP_\hstar$. In round $t$, let $S_t$ denote the set of $M$ experts chosen by the algorithm to query.

Lemma~\ref{lem:high-regret-rounds} shows that if either of the events $[\hstar \notin S_t]$ or $[\hstar \in S_t,\ A_t \neq \hstar(t)]$ happens, the algorithm suffers an expected regret of at least $\eps/2$. Define the random variables
\[ L_\hstar = \sum_{t=1}^T \ind[\hstar \in S_t] \quad \text{and} \quad N_\hstar = \sum_{t=1}^T \ind[\hstar \in S_t, A_t = \hstar(t)].\]
Then to get a lower bound on the expected regret we need to upper bound $\E_\hstar[N_\hstar]$. To do this, we use the usual arguments based on KL-divergence between the distributions $\bP_\hstar$ and $\bP_0$. Specifically, for all $t$, let 
\[H_t = \langle (G_1, \ell_1(A_1)), (G_2, \ell_2(A_2)), \ldots, (G_t, \ell_t(A_t))\rangle\]
denote the history up to time $t$; here, $G_\tau$ is the vector of advices of the experts queried at time $\tau$, viz. the experts in $S_\tau$. For convenience, we define $H_0 = \langle \rangle$, the empty vector. Note that since the algorithm is assumed to be deterministic, $N_\hstar$ is a deterministic function of the history $H_T$. Thus to upper bound $\E_\hstar[N_\hstar]$ we compute an upper bound on $\kl{\bP_0(H_T)}{\bP_\hstar(H_T)}$. Lemma~\ref{lem:kl-upper-bound} shows that 
\[\kl{\bP_0(H_T)}{\bP_\hstar(H_T)}\ \leq\ 6\eps^3 \E_0[N_\hstar] + \frac{4\eps^2}{K^2}\E_0[L_\hstar].\] Thus, by Pinsker's inequality, we get
\[ \dtv{\bP_0(H_T)}{\bP_\hstar(H_T)}\ \leq\ \sqrt{\half \kl{\bP_0(H_T)}{\bP_\hstar(H_T)}}\ \leq\ \sqrt{3\eps^2\E_0[N_\hstar] + 2\frac{\eps^2}{K^2}\E_0[L_\hstar]}.\]
Since $|N_\hstar| \leq T$, this implies that
\[\E_\hstar[N_\hstar]\ \leq\ \E_0[N_\hstar] + T\sqrt{3\eps^2\E_0[N_\hstar] + 2\frac{\eps^2}{K^2}\E_0[L_\hstar]}.\]
By Jensen's inequality applied to the concave square root function, we get
\begin{align}
\frac{1}{N}\sum_{\hstar \in \mN}\E_\hstar[N_\hstar]\ &\leq\ \frac{1}{N}\sum_{\hstar \in \mN}\E_0[N_\hstar] + T\sqrt{3\eps^2\bigB{\frac{1}{N}\sum_{\hstar \in \mN}\E_0[N_\hstar]} + 2\frac{\eps^2}{K^2}\bigB{\frac{1}{N}\sum_{\hstar \in \mN}\E_0[L_\hstar]}} \notag \\
&\leq\ \frac{f(K, M)}{N}T + T\sqrt{3\eps^2\frac{f(K, M)}{N}T + 2\frac{\eps^2 M}{K^2 N}T} \label{eq:upper-bounds}\\
&\leq\ 4\eps T\sqrt{\frac{f(K, M)}{N}T} \label{eq:simplification}.
\end{align}
Inequality (\ref{eq:upper-bounds}) follows from Lemma~\ref{lem:balls-bins-upper-bound} using
\[\sum_{\hstar \in \mN} \E_0[L_\hstar] = \sum_{t=1}^T \sum_{\hstar \in \mN} \bP_0[\hstar \in S_t]\ \leq\ MT\] and 
\[\sum_{\hstar \in \mN} \E_0[N_\hstar]\ =\ \sum_{t=1}^T \sum_{\hstar \in \mN} \bP_0[\hstar \in S_t, A_t = \hstar(t)]\ \leq\ f(K, M)T.\] Inequality (\ref{eq:simplification}) follows because $f(K, M)$ is at least the expected number of balls in each bin, which equals $\frac{M}{K}$, and so $f(K, M) \geq \frac{M}{K^2}$.
Now, taking expectation over the choice of the expert $\hstar$, the expected regret of the algorithm is at least
\begin{align*}
	\frac{1}{N}\sum_{\hstar \in \mN}\frac{\eps}{2}(T - \E_\hstar[N_\hstar])\ &\geq\ \frac{\eps}{2}T -  2\eps^2 T\sqrt{\frac{f(K, M)}{N}T}\\
	&=\ \frac{1}{32}\sqrt{\frac{N}{f(K, M)}T}\ =\ \Omega\bigP{\sqrt{\frac{\min\{K, \frac{M}{\log(K)}\} N}{M}T}},
\end{align*}
using the setting $\eps = \frac{1}{8}\sqrt{\frac{N}{f(K, M)T}}$ and the fact that $f(K, M) = O(\max\{\log(K), \frac{M}{K}\})$.
\end{proof}

\begin{Lem} \label{lem:high-regret-rounds}
	Suppose $\hstar$ is the expert chosen in the beginning and advices and losses are then generated from $\bP_\hstar$. Then in any round $t$, if either of the events $[\hstar \notin S_t]$ or $[\hstar \in S_t,\ A_t \neq \hstar(t)]$ happens, the algorithm suffers an expected regret of at least $\eps/2$.
\end{Lem}
\begin{proof}
	First, recall that the expert $\hstar$ always incurs an expected loss of $\half - \eps$ in each round $t$.

	Now if $\hstar \notin S_t$, then the losses of the arms are independent of the advices of the experts in $S_t$, and hence their distribution {\em conditioned on the advices of experts in $S_t$} is $\bernoulli(p)$. This conditioning is important since the algorithm chooses the arm to play, $A_t$, based on the advice of the experts in $S_t$. Thus, the distribution of the chosen arm $A_t$ is also $\bernoulli(p)$, which implies that the algorithm suffers an expected regret of $p - (\half - \eps) = \eps(1 - 1/K) \geq \eps/2$.

	If $\hstar \in S_t$ but $A_t \neq \hstar(t)$, then the distribution of the loss of $A_t$, conditioned on the advices of the experts in $S_t$, is $\bernoulli(\half)$. This implies that the algorithm suffers an expected regret of $\half - (\half - \eps) = \eps \geq \eps/2$.
\end{proof}

\begin{Lem} \label{lem:kl-upper-bound}
	We have
	\[\kl{\bP_0(H_T)}{\bP_\hstar(H_T)}\ \leq\ 6\eps^3 \E_0[N_\hstar] + \frac{4\eps^2}{K^2}\E_0[L_\hstar].\]
\end{Lem}
\begin{proof}
	We have
\begin{align}
	\kl{\bP_0(H_T)}{\bP_\hstar(H_T)}\ &=\ \sum_{t=1}^T \kl{\bP_0((G_t, \ell_t(A_t)) | H_{t-1})}{\bP_\hstar((G_t, \ell_t(A_t)) | H_{t-1})} \label{eq:chain-rule1}\\
	&=\ \sum_{t=1}^T [\kl{\bP_0(\ell_t(A_t) | H_{t-1}, G_t)}{\bP_\hstar(\ell_t(A_t) | H_{t-1}, G_t)} \notag \\ 
	& \qquad \qquad + \kl{\bP_0(G_t | H_{t-1})}{\bP_\hstar(G_t | H_{t-1})}] \label{eq:chain-rule2}\\
	&=\ \sum_{t=1}^T \kl{\bP_0(\ell_t(A_t) | H_{t-1}, G_t)}{\bP_\hstar(\ell_t(A_t) | H_{t-1}, G_t)}\label{eq:drop-G_t}\\
	&=\ \sum_{t=1}^T \bP_0[\hstar \in S_t, A_t = \hstar(t)] \kl{\bernoulli(p)}{\bernoulli(\half - \eps)} \notag \\ 
	& \qquad + \bP_0[\hstar \in S_t, A_t \neq \hstar(t)] \kl{\bernoulli(p)}{\bernoulli(\half)} \notag \\
	& \qquad + \bP_0[\hstar \notin S_t] \kl{\bernoulli(p)}{\bernoulli(p)} \label{eq:kl-bernoulli}\\
	&\leq\ \sum_{t=1}^T \bP_0[\hstar \in S_t, A_t = \hstar(t)] \cdot 6\eps^2 + \bP_0[\hstar \in S_t, A_t \neq \hstar(t)]  \cdot \frac{4\eps^2}{K^2} \label{eq:kl-bernoulli-ub}\\
	&\leq\ \sum_{t=1}^T 6\eps^2\bP_0[\hstar \in S_t, A_t = \hstar(t)] + \frac{4\eps^2}{K^2}\bP_0[\hstar \in S_t] \notag\\
	&=\ 6\eps^3 \E_0[N_\hstar] + \frac{4\eps^2}{K^2}\E_0[L_\hstar] \notag.
\end{align}
Equalities (\ref{eq:chain-rule1}) and (\ref{eq:chain-rule2}) follow from the chain rule for relative entropy. Equality (\ref{eq:drop-G_t}) follows because the distribution of $G_t$ conditioned on $H_{t-1}$ is identical in $\bP_0$ and $\bP_\hstar$. Equality (\ref{eq:kl-bernoulli}) follows because if $\hstar \notin S_t$, then the loss of the chosen arm follows $\bernoulli(p)$, if $\hstar \in S_t$ and $A_t = \hstar(t)$, then the loss of the chosen arm follows $\bernoulli(\half - \eps)$, and if $\hstar \in S_t$ and $A_t \neq \hstar(t)$, then the loss of the chosen arm follows $\bernoulli(\half)$. Finally, inequality (\ref{eq:kl-bernoulli-ub}) follows using standard calculations for KL-divergence between Bernoulli random variables.
\end{proof}

\begin{Lem} \label{lem:balls-bins-upper-bound}
	Recall that $f(K, M)$ is the expected number of balls in the bin with the maximum balls in a $M$-balls-into-$K$-bins process. Then for all $t$,
	\[\sum_{\hstar \in \mN} \bP_0[\hstar \in S_t]\ =\ M \quad \text{and} \quad \sum_{\hstar \in \mN} \bP_0[\hstar \in S_t, A_t = \hstar(t)]\ \leq\ f(K, M).\]
\end{Lem}
\begin{proof}
First, we have
\[\sum_{\hstar \in \mN} \bP_0[\hstar \in S_t]\ =\ \E_0\bigB{\sum_{\hstar \in \mN}\ind[\hstar \in S_t]}\ =\ \E_0[|S_t|]\ =\ M.\]
Next, we have
\begin{align*}
	\sum_{\hstar \in \mN} \bP_0[\hstar \in S_t, A_t = \hstar(t)]\ &=\ \E_0\bigB{\sum_{\hstar \in \mN}\ind[\hstar \in S_t, A_t = \hstar(t)]} \\
	&=\ \E_0\bigB{|\{\hstar \in S_t:\ A_t = \hstar(t)\}|} \\
	&\leq\ \E_0\bigB{\max_{a \in \mK}\{|\{\hstar \in S_t:\ a = \hstar(t)\}|\}} \\
	&=\ \E_0\bigB{\E_0\bigB{\max_{a \in \mK}\{|\{\hstar \in S_t:\ a = \hstar(t)\}|\}\ |\ S_t}} \\
	&=\ \E_0[f(K, M)]\\
	&=\ f(K, M).
\end{align*}
The penultimate equality follows because conditioning on the choice of $S_t$, the random variable $\max_{a \in \mK}\{|\{\hstar \in S_t:\ a = \hstar(t)\}|\}$ is completely determined by the choice of the recommended arms for the experts $\hstar \in S_t$. Since these arms are chosen uniformly at random from $\mK$ independently for each expert $\hstar \in S_t$, we can think of the $M$ experts in $S_t$ as ``balls'' and the $K$ arms in $\mK$ as ``bins'' in a balls-into-bins process. Then the random variable of interest is exactly the number of balls in the bin with maximum number of balls. The expectation of this random variable is $f(K, M)$.
\end{proof}

\section*{Acknowledgments}
The author thanks Elad Hazan, Dean Foster, Rob Schapire, and Yevgeny Seldin for discussions on this problem.

\bibliographystyle{plainnat}
\bibliography{limbandits}
\end{document}